\PassOptionsToPackage{table}{xcolor}
\documentclass[sigconf, nonacm]{acmart}
\usepackage{enumitem}
\usepackage{amsmath,amsfonts,bm}

\def\1{\bm{1}}

\DeclareMathAlphabet{\mathsfit}{\encodingdefault}{\sfdefault}{m}{sl}
\SetMathAlphabet{\mathsfit}{bold}{\encodingdefault}{\sfdefault}{bx}{n}

\usepackage{adjustbox}
\usepackage{fontawesome5}
\usepackage{multirow}
\usepackage{array}
\usepackage{tabularx}
\usepackage{colortbl}
\usepackage{xspace}
\definecolor{focuscol}{RGB}{237,237,237}
\definecolor{oursrow}{RGB}{220,237,247}
\definecolor{oursink}{RGB}{27,78,112}

\definecolor{qwenmodel}{HTML}{5D4E8C}
\definecolor{llamamodel}{HTML}{CC6A50}
\colorlet{qwenrow}{qwenmodel!15!white}
\colorlet{llamarow}{llamamodel!15!white}

\definecolor{trainfire}{RGB}{190,65,35}
\definecolor{frozencold}{RGB}{35,105,150}

\newcolumntype{F}{>{\columncolor{focuscol}}c}

\newcommand{\system}{\textsc{CoSTA}\xspace}

\hypersetup{
    colorlinks=true,
    urlcolor=red,
    linkcolor=red,
    citecolor=blue
}
\setcopyright{none}
\begin{document}

\title[CoSTA: Co-Evolving Self-Play Timeline Agents for Complex Temporal Knowledge Graph Question Answering]{\system: Co-Evolving Self-Play Timeline Agents for Complex Temporal Knowledge Graph Question Answering}

\author{Haikun Liu}
\affiliation{
  \institution{Shanxi University}
  \city{Taiyuan}
  \country{China}
}
\email{liuhaikun1@sxu.edu.cn}

\author{Zhichao Yan, Shaoru Guo}
\affiliation{
  \institution{Shanxi University}
  \city{Taiyuan}
  \country{China}
}
\email{\{yanzhichao, gsr\}@sxu.edu.cn}

% \author{Shaoru Guo}
% \affiliation{
%   \institution{Shanxi University}
%   \city{Taiyuan}
%   \country{China}
% }
% \email{gsr@sxu.edu.cn}

\author{Jiawei Gu}
\affiliation{
  \institution{Singapore University of Technology and Design}
  \country{Singapore}
}
\email{jiawei\_gu@sutd.edu.sg}

\author{Víctor Gutiérrez-Basulto}
\affiliation{
  \institution{Cardiff University}
  \city{Cardiff}
  \country{UK}
}
\email{GutierrezBasultoV@cardiff.ac.uk}

\author{Jiapu Wang}
\authornote{Corresponding author.}
\affiliation{
  \institution{Nanjing University of Science and Technology}
  \city{Nanjing}
  \country{China}
}
\email{jiapu.wang@njust.edu.cn}

\author{Ru Li}
\affiliation{
  \institution{Shanxi University}
  \city{Taiyuan}
  \country{China}
}
\email{liru@sxu.edu.cn}

\begin{abstract}
Complex Temporal Knowledge Graph Question Answering (TKGQA) involves more than retrieving facts: the timestamp or validity interval of one fact can constrain the search for another, while multiple intervals may jointly determine the answer. Existing single-agent TKGQA methods typically investigate these facts within a single reasoning context, where evidence from different event histories can become interleaved and distractors can misdirect subsequent retrieval. Distributing investigations across agents can separate event-specific evidence into distinct contexts, but effective collaboration must still be learned because question decomposition, subtask investigation, and evidence synthesis are interdependent. To address these limitations, we propose \textbf{Co}-Evolving \textbf{S}elf-Play \textbf{T}imeline \textbf{A}gents (CoSTA), which organizes facts into navigable timelines, delegates investigations to workers, and uses a coordinator to synthesize their original timestamped observations. CoSTA initializes a shared policy through role-conditioned distillation and jointly optimizes both roles through co-evolving self-play with final-answer rewards. Across three benchmarks, CoSTA achieves state-of-the-art overall performance, with particularly strong gains on complex multi-fact questions, while its compact variants outperform the evaluated closed-source LLM agents\footnote{Code and data are available at: \href{https://github.com/CoSTA.git}{https://github.com/haikunhere/CoSTA}}.
\end{abstract}

\begin{CCSXML}
<ccs2012>
   <concept>
       <concept_id>10010147.10010178.10010187.10010193</concept_id>
       <concept_desc>Computing methodologies~Temporal reasoning</concept_desc>
       <concept_significance>500</concept_significance>
       </concept>
   <concept>
       <concept_id>10010147.10010178.10010219.10010220</concept_id>
       <concept_desc>Computing methodologies~Multi-agent systems</concept_desc>
       <concept_significance>500</concept_significance>
       </concept>
 </ccs2012>
\end{CCSXML}

\ccsdesc[500]{Computing methodologies~Temporal reasoning}
\ccsdesc[500]{Computing methodologies~Multi-agent systems}

\keywords{Temporal Knowledge Graph, Knowledge Graph Question Answering, Multi-Agent Collaboration, Co-Evolving Self-Play}

\maketitle

\section{Introduction}
\label{sec:introduction}

\begin{figure}[t]
  \centering
  \includegraphics[width=0.95\columnwidth]{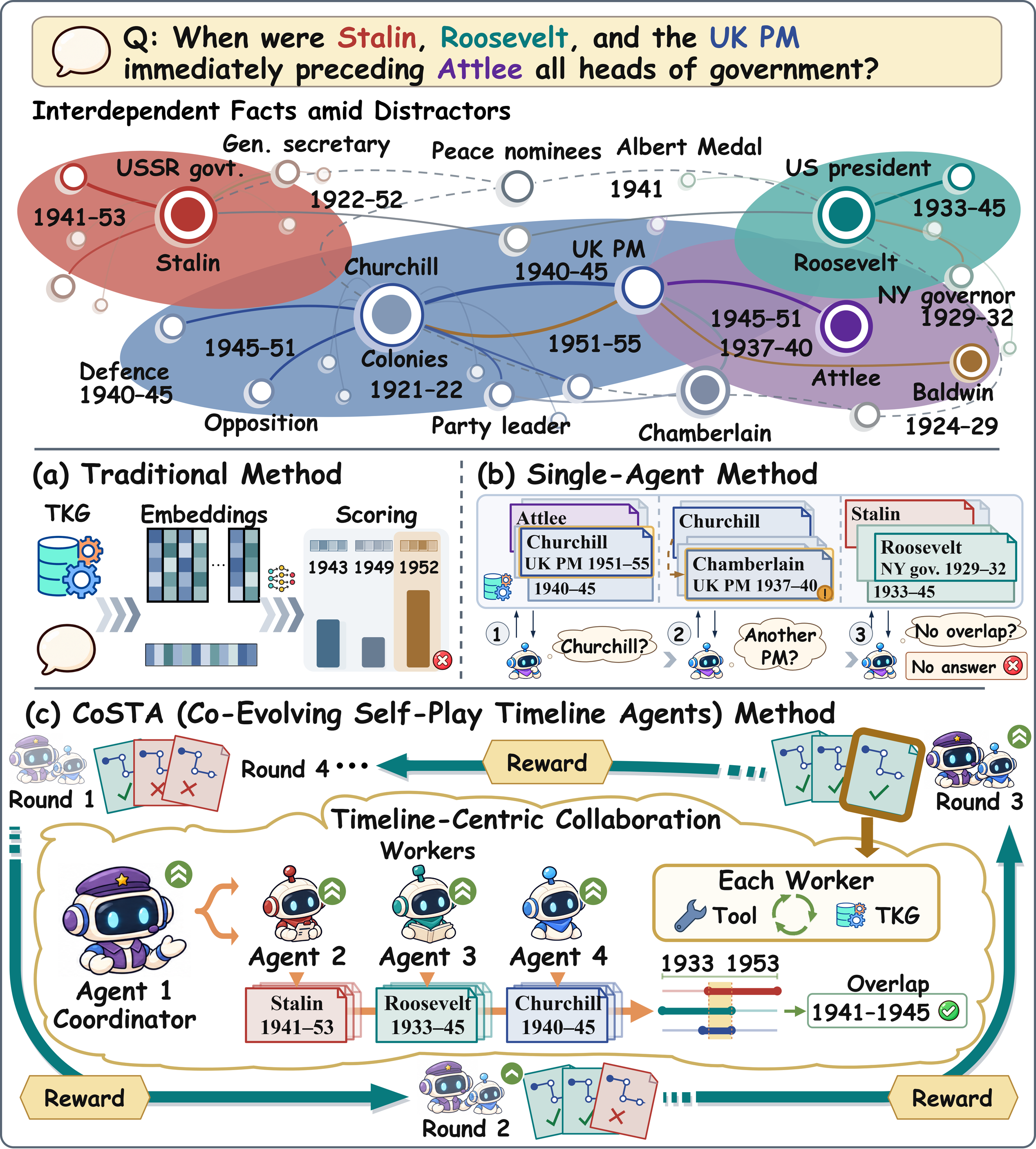}
  \caption{Motivating comparison of (a) Traditional TKGQA, (b) Single-agent TKGQA, and (c) Our proposed \system{}, which coordinates separate timeline investigations and co-evolves the coordinator and workers through self-play.}
  \label{fig:intro}
\end{figure}

Factual knowledge on the Web evolves over time, with collaborative knowledge bases and news media continuously recording changes in entity relations and new events~\citep{vrandecic2014wikidata,chen2023multitq}.
However, conventional knowledge graphs typically treat facts as static, although many real-world relations are valid only within specific periods~\citep{hogan2021knowledge}.
Temporal Knowledge Graphs (TKGs) address this limitation by associating facts with timestamps or validity intervals, thereby capturing the temporal validity and evolution of such relations~\citep{wang2024large,sun2025timelinekgqa,ime2024,qdn2023}.
Building on TKGs, Temporal Knowledge Graph Question Answering (TKGQA) is proposed, which aims to answer natural-language questions by reasoning over temporally qualified facts~\citep{saxena2021cronkgqa,chen2023multitq}.
%By targeting reliable question answering over evolving Web knowledge, our work aligns with the \textbf{Web Mining, Multimedia and Multilingual Content Analysis track}.

However, complex temporal questions may involve dependencies among multiple facts: the time of one fact can constrain the retrieval of another~\citep{chen2023multitq}, while the intervals of several facts may jointly determine the answer~\citep{sun2025timelinekgqa}. For example, the question in Figure~\ref{fig:intro} first requires identifying the \texttt{UK PM} immediately preceding \texttt{Attlee} and then intersecting the tenure intervals of \texttt{Stalin}, \texttt{Roosevelt}, and the identified prime minister.

Such dependencies make complex TKGQA a coordination problem rather than merely a longer reasoning process. When evidence from multiple event histories is accumulated within a single reasoning context, observations, timestamps, and intermediate search steps can become interleaved. This entanglement obscures fact--time correspondences and increases the risk that distractor facts derail subsequent retrieval, as illustrated by the single-agent failure in Figure~\ref{fig:intro}(b).

To separate these event histories without losing their temporal dependencies, we draw inspiration from transactive memory systems, in which human teams distribute knowledge across members while coordinating its retrieval~\citep{wegner1995computer,lewis2003measuring}.
This analogy suggests distributing event-specific investigations while coordinating their evidence under global temporal constraints.
Yet specialization alone does not ensure effective collaboration: a controlled study of human teams found that jointly trained teams outperformed those assembled from individually trained members, with the performance advantage mediated primarily by the development of transactive memory~\citep{liang1995group}.

Existing TKGQA methods optimize individual-agent reasoning and tool use~\citep{gong2026tempr1,huang2026temptoolr1,jiang2026tkgthinker}, but their single-agent objectives do not capture dependencies across collaborative roles.
These roles are mutually dependent: the value of a worker's investigation depends on both the assigned subtask and how its evidence is synthesized, whereas the coordinator is constrained by what workers retrieve and report.
Optimizing either role in isolation may therefore favor behavior that is locally effective but fails to improve final-answer quality.
This calls for team-level learning from end-to-end collaborative outcomes, following the broader principle of co-evolving complementary roles through interaction~\citep{wang2025cure}.

To address these challenges, we propose \textbf{Co}-Evolving \textbf{S}elf-Play \textbf{T}imeline \textbf{A}gents (\textbf{\system}), a timeline-centric collaborative framework for complex TKGQA.
\system organizes TKG facts into chronologically ordered timelines, allowing agents to enter relevant timelines through semantic retrieval, inspect temporal neighborhoods through timeline reading, and traverse to related timelines via shared facts.
A coordinator decomposes each question into interdependent subtasks, which workers investigate in separate contexts before returning their original timestamped tool observations together with concise reports relating the evidence to each subtask.
The coordinator integrates this evidence under the temporal constraints of the full question, for example, by ordering facts or intersecting intervals, and launches follow-up investigations when necessary.

Beyond this inference architecture, \system instantiates the coordinator and workers with a shared policy conditioned on role-specific prompts~\citep{mo2026matpo}.
Role-conditioned distillation from successful collaborative demonstrations first initializes the policy with the reasoning and tool-use behaviors required by each role.
Co-evolving self-play then optimizes the shared policy over complete collaborative trajectories, using cross-role credit assignment to propagate the final-answer reward to every participating role.
Experiments on three TKGQA benchmarks show that \system outperforms prior methods, particularly on complex questions; further analyses validate the contributions of multi-agent collaboration and co-evolving self-play.

The main contributions of this paper are summarized as follows:
\begin{itemize}[left=0cm]
    
 \item \textbf{Timeline-centric collaboration.}
    To the best of our knowledge, we introduce the first jointly trained coordinator--worker framework for complex TKGQA, separating worker-side timeline investigation from coordinator-side synthesis of the original temporal evidence.

    \item \textbf{Co-evolving self-play.}
    We develop a shared-policy strategy that initializes both roles through role-conditioned distillation and jointly optimizes delegation, investigation, and synthesis through self-play with cross-role credit assignment from final-answer feedback.

    \item \textbf{Empirical validation.}
    We show on three benchmarks that CoSTA achieves the best overall performance, especially on multi-fact questions; ablations validate timeline navigation and evidence-preserving collaboration, while compact variants outperform evaluated closed-source LLM agents.

\end{itemize}

\section{Related Work}
\label{sec:related-work}

\begin{figure*}[t]
  \centering
  \includegraphics[width=0.95\textwidth]{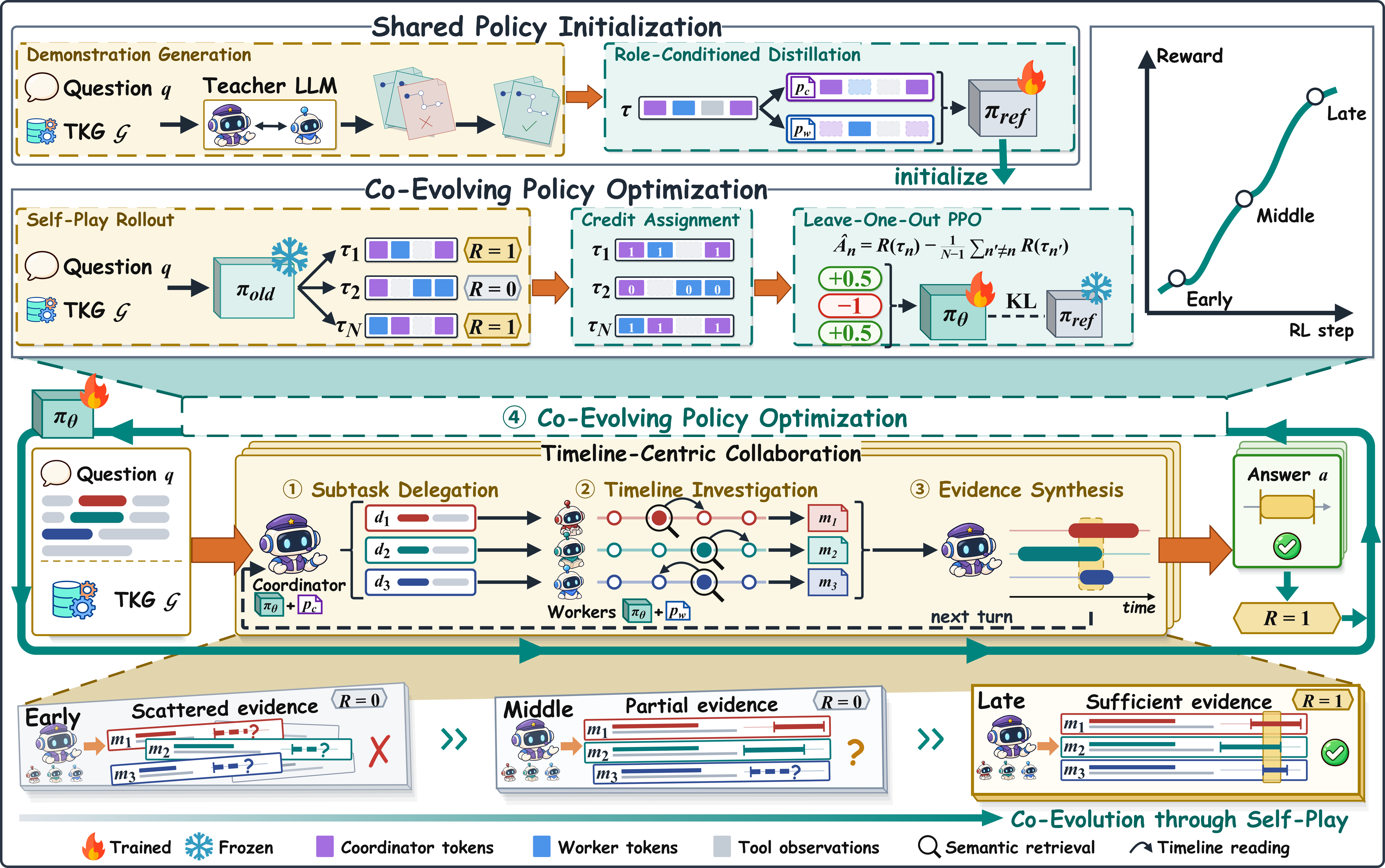}
  \caption{The framework of \system{}, which initializes a shared policy for the coordinator and workers from teacher demonstrations and optimizes it through self-play over timeline-centric collaboration (\ding{172}--\ding{175}), enabling the two roles to co-evolve. The reward curve is illustrative.}
  \label{fig:framework}
\end{figure*}

Existing TKGQA methods can be roughly categorized into retrieval-based and agent-based methods.
The two categories differ in whether the reasoning steps are fixed in advance or decided by an agent based on intermediate observations.

\subsection{Retrieval-based TKGQA Methods}

Early TKGQA methods are mainly embedding-based or semantic parsing-based~\cite{su2024tkgqasurvey}. CronKGQA~\cite{saxena2021cronkgqa} scores candidate entities and timestamps using TKG embeddings~\cite{multi2023, tdn2023, multilevel2023}.
Based on CronKGQA, TempoQR~\cite{mavromatis2022tempoqr} further grounds question representations in relevant entities and time scopes.
To reason over multiple related facts, EXAQT~\cite{jia2021complex}, TwiRGCN~\cite{sharma2023twirgcn}, and LGQA~\cite{liu2023lgqa} apply graph neural networks~\cite{wang2026comgrl, zhao2026hybrid} to question-relevant temporal subgraphs.
Instead of implicit scoring, semantic parsing-based methods express temporal constraints explicitly.
TEQUILA~\cite{jia2018tequila} decomposes a question into sub-questions and temporal constraints, while SF-TQA~\cite{ding2022sftqa} and Prog-TQA~\cite{chen2024progtqa} translate questions into query graphs and executable programs, respectively~\cite{zhuo2025effective, ijcai2025p413}.
However, these methods rely on task-specific representations or predefined constraint types and operators, making it difficult to cover the diverse temporal constraints in questions~\cite{su2024tkgqasurvey}.

LLM-enhanced retrieval methods leverage the language understanding and reasoning abilities of LLMs to interpret these constraints directly over retrieved facts.
GenTKGQA~\cite{gao2024gentkgqa} incorporates graph representations of temporal subgraphs into an LLM.
To handle implicit temporal constraints, TimeR4~\cite{qian2024timer4} rewrites them into explicit ones before time-aware retrieval.
PoK~\cite{qian2025pok} further plans retrieval, temporal ranking, and reasoning as a sequence of sub-objectives.
Meanwhile, RTQA~\cite{gong2025rtqa} recursively solves decomposed sub-questions using their intermediate answers.
To organize the evidence, TECQA~\cite{liu2026tecqa} retains temporal neighbors around the temporal anchors of a question.
Similarly, SCoP~\cite{guo2026scop} compiles temporal constraints into executable evidence filters.
Nevertheless, these methods follow predefined pipelines or plans generated in advance, rather than deciding each step based on intermediate observations.

\subsection{Agent-based TKGQA Methods}

To decide each step based on intermediate observations, agent-based methods employ LLM agents that iteratively invoke tools on TKGs.
TempAgent~\cite{hu2025tempagent} equips a ReAct~\cite{yao2023react} agent with temporal retrieval tools at multiple granularities.
Beyond a single trajectory, ARI~\cite{chen2024ari} induces reusable reasoning guidance from historical trajectories.
Similarly, MemoTime~\cite{tan2026memotime} reuses past reasoning through an experience memory and organizes sub-questions into a Tree of Time.
To improve the reliability of reasoning, SAR~\cite{hu2026sar} aligns retrieval with the TKG schema and reorders retrieved facts chronologically.
Instead of refining retrieval, Regret-Now~\cite{huang2026regretnow} detects abnormal reasoning steps and rolls back to earlier decisions.
However, these agents rely on prompting and stored experience, without training the underlying LLMs with interaction feedback.

RL-based agents further train the underlying LLMs with interaction feedback.
MRE~\cite{wen2026mre} optimizes multi-hop decisions with feedback from tree-structured exploration.
Temp-R1~\cite{gong2026tempr1} adds internal reasoning actions and trains the agent with a reverse curriculum.
To supervise tool interactions, TempTool-R1~\cite{huang2026temptoolr1} introduces order-sensitive tool rewards, while TKG-Thinker~\cite{jiang2026tkgthinker} combines answer, format, and retrieval rewards.
However, these agents investigate all required facts within a single reasoning context, where evidence from different event histories can become interleaved and distractors can misdirect subsequent retrieval.
They are also optimized as a single agent rather than as multiple interdependent roles.

Different from the aforementioned TKGQA methods, the proposed \system{} is a timeline-centric collaborative framework that learns collaboration among agents for complex TKGQA.
By distributing the investigations of different event histories across workers in separate contexts, \system{} keeps their evidence from interleaving, while co-evolving self-play jointly optimizes the workers and their coordinator as interdependent roles from final-answer rewards.

\section{Preliminary}
\label{sec:preliminary}

\noindent\textbf{Temporal Knowledge Graph (TKG).}
A TKG is represented as $\mathcal{G}=\{\mathcal{E},\mathcal{R},\mathcal{T},\mathcal{Q}\}$, where $\mathcal{Q}$ is a set of quadruplets over the entity set $\mathcal{E}$, the relation set $\mathcal{R}$, and the timestamp set $\mathcal{T}$.
Each quadruplet $(e_s,r,e_o,t)\in\mathcal{Q}$ denotes the fact that subject entity $e_s$ has relation $r$ with object entity $e_o$ at time $t$, where $t$ is either a timestamp in $\mathcal{T}$ or a time interval bounded by two timestamps~\cite{chen2023multitq,sun2025timelinekgqa}.

\noindent\textbf{Timeline.}
To read each fact among its temporal neighbors, we organize the facts of $\mathcal{G}$ into chronologically ordered timelines.
For an entity $e$, a relation $r$, and $\rho\in\{s,o\}$ marking subject or object, the timeline of $e$ collects the facts in which $e$ appears as $e_\rho$ under $r$ and orders them by time:
\begin{equation}
\mathcal{L}^{\rho}_{e,r}=\operatorname{sort}_{t}\!\left(\{f=(e_s,r,e_o,t)\in\mathcal{Q}\mid e_\rho=e\}\right),
\label{eq:timeline}
\end{equation}
where $\operatorname{sort}_{t}$ orders facts by timestamp or interval start.

\noindent\textbf{Temporal Knowledge Graph Question Answering (TKGQA).}
Given a TKG $\mathcal{G}$ and a natural-language question $q$, TKGQA aims to predict the answer set $\mathcal{A}_q$ of $q$ from the facts in $\mathcal{G}$, where each answer $a \in \mathcal{A}_q$ is either an entity or a temporal value~\cite{saxena2021cronkgqa,chen2023multitq,sun2025timelinekgqa}.
The facts required to derive $\mathcal{A}_q$ are called the \textit{supporting facts} of $q$~\cite{sun2025timelinekgqa}.
As the supporting facts are not given with $q$, we formulate TKGQA with a retriever $P(\mathcal{C} \mid q, \mathcal{G})$ that selects a fact set $\mathcal{C} \subseteq \mathcal{Q}$ and a reasoner $P(a \mid q, \mathcal{C})$ that derives the answers from $\mathcal{C}$:
\begin{equation}
P(a \mid q, \mathcal{G}) = \sum_{\mathcal{C} \subseteq \mathcal{Q}} P(a \mid q, \mathcal{C})\, P(\mathcal{C} \mid q, \mathcal{G}).
\label{eq:tkgqa}
\end{equation}
Since the number of fact sets grows exponentially with $|\mathcal{Q}|$, Eq.~(\ref{eq:tkgqa}) is approximated by its term for the fact set $\mathcal{C}^*$ retrieved for $q$.
As $P(\mathcal{C}^* \mid q, \mathcal{G})$ is the same for every candidate answer, $P(a \mid q, \mathcal{G}) \propto P(a \mid q, \mathcal{C}^*)$~\cite{peng2024graph}.

\noindent\textbf{Multi-Agent Collaboration.}
Multi-agent collaboration solves complex tasks by organizing a coordinator and multiple workers~\cite{mo2026matpo,ke2026masorchestra}.
Specifically, the coordinator decomposes a question $q$ into subtasks and delegates them to workers, each of which solves its subtask through multi-turn tool use within a separate context and returns a message.
Based on the returned messages, the coordinator either delegates further subtasks or derives the answer $a$, which can be formulated as:
\begin{equation}
\begin{aligned}
o_k &\sim \pi_{\mathrm{c}}(\cdot \mid q, \mathcal{H}_{<k}), && o_k \in \{\mathcal{D}_k, a\},\\
m_{k,j} &\sim \pi_{\mathrm{w}}(\cdot \mid d_{k,j}, \mathcal{G}), && d_{k,j} \in \mathcal{D}_k,
\end{aligned}
\label{eq:mac}
\end{equation}
where $o_k$ is the output of the coordinator at turn $k$, which is either a set of subtasks $\mathcal{D}_k$ or the answer $a$, and $\mathcal{H}_{<k}$ denotes the responses of the coordinator before turn $k$ together with the messages returned for their subtasks.
The message $m_{k,j}$ returned for subtask $d_{k,j}$ is typically the answer of the worker to this subtask~\cite{mo2026matpo,ke2026masorchestra}.
The coordinator and worker policies $\pi_{\mathrm{c}}$ and $\pi_{\mathrm{w}}$ can be instantiated by different LLMs or by a shared LLM with role-specific system prompts~\cite{mo2026matpo}.

\section{Methodology}
\label{sec:method}

In this section, we present \system, which consists of three components: \textit{Timeline-Centric Collaboration} (Section~\ref{sec:collaboration}) enables a coordinator and workers to collaboratively answer temporal questions; \textit{Shared Policy Initialization} (Section~\ref{sec:initialization}) initializes their shared policy through demonstration-based training; and \textit{Co-Evolving Policy Optimization} (Section~\ref{sec:selfplay}) further improves the shared policy through self-play, enabling the coordinator and worker roles to co-evolve. Figure~\ref{fig:framework} provides an overview of the framework.

\subsection{Timeline-Centric Collaboration}
\label{sec:collaboration}

Timeline-centric collaboration answers temporal questions with a coordinator and multiple workers.
It consists of three steps: Subtask Delegation, Timeline Investigation, and Evidence Synthesis.

\noindent\textbf{Subtask Delegation.}
Subtask delegation decomposes a question into subtasks that target different candidate timelines and assigns them to separate workers.
\system uses a shared policy $\pi_\theta$ for both roles, distinguished by role-specific system prompts $p_{\mathrm{c}}$ and $p_{\mathrm{w}}$.
At turn $k$, the coordinator generates a response conditioned on the question and collaboration history:
\begin{equation}
y^{\mathrm{c}}_k \sim \pi_\theta(\cdot \mid p_{\mathrm{c}}, q, \mathcal{H}_{<k}).
\label{eq:delegation}
\end{equation}
As defined in Eq.~(\ref{eq:mac}), the coordinator's output $o_k$ is either a set of subtasks $\mathcal{D}_k$ or the final answer $a$.
Each subtask $d_{k,j}\in\mathcal{D}_k$, such as identifying the UK prime minister immediately preceding Attlee, is assigned to a worker with its own context.
Workers investigate subtasks independently within each turn, while subsequent turns can build on returned messages to delegate additional subtasks.

\noindent\textbf{Timeline Investigation.}
Timeline investigation explores the candidate timelines associated with each subtask and gathers temporal evidence through two tools: semantic retrieval and timeline reading.
Following~\cite{mo2026matpo}, the worker receives both the original question $q$ and the assigned subtask $d_{k,j}$, and generates a response at step $u$ based on its earlier responses and tool observations $\xi_{k,j,<u}$:
\begin{equation}
y^{\mathrm{w}}_{k,j,u} \sim \pi_\theta(\cdot \mid p_{\mathrm{w}}, q, d_{k,j}, \xi_{k,j,<u}).
\label{eq:investigation}
\end{equation}
Each response either invokes a tool or ends the investigation.
Semantic retrieval returns the facts most similar to a query as entry points for timeline exploration.
Timeline reading uses an observed fact as a temporal anchor and returns neighboring facts from its subject or object timeline, as defined in Eq.~(\ref{eq:timeline}).
By following facts shared across timelines, the worker can navigate between related timelines.
Since each timeline is chronologically ordered, the worker can directly identify facts immediately before or after the anchor, such as the prime minister preceding Attlee.
Finally, the worker returns a message $m_{k,j}=(\mathcal{O}_{k,j},z_{k,j})$, where $\mathcal{O}_{k,j}$ contains all tool observations with their original timestamps, and $z_{k,j}$ is its report summarizing the investigation findings.

\begin{table*}[!t]
\centering
\captionsetup{skip=5pt}
\caption{Hits@1 (\%) on MultiTQ, TimelineKGQA-CronQuestion, and TimelineKGQA-ICEWS-Actor. Baselines are grouped into retrieval-based methods, prompting-based agents, and RL-based agents. Bold and underlined values denote the best and second-best scores, respectively.}
\label{tab:main_results}

\small
\renewcommand{\arraystretch}{1.08}
\setlength{\tabcolsep}{2.7pt}

\begin{adjustbox}{max width=\textwidth}
\begin{tabular}{l c c F @{\hspace{7pt}} c c F F @{\hspace{7pt}} c c F F}
\toprule
\multirow{2}{*}{Method}
& \multicolumn{3}{c}{MultiTQ}
& \multicolumn{4}{c}{TimelineKGQA-CronQuestion}
& \multicolumn{4}{c}{TimelineKGQA-ICEWS-Actor}
\\
\cmidrule(lr){2-4}
\cmidrule(lr){5-8}
\cmidrule(lr){9-12}
&
Overall
& Single
& \textbf{Multiple}
& Overall
& Simple
& \textbf{Medium}
& \textbf{Complex}
& Overall
& Simple
& \textbf{Medium}
& \textbf{Complex}
\\
\midrule

TimeR$^4$ (2024)~\citep{qian2024timer4}
& 72.8 & 88.7 & 33.5
& 40.5 & 77.8 & 24.2 & 26.4
& 36.5 & 58.0 & 22.2 & 20.4 \\

PoK (2026)~\citep{qian2025pok}
& 77.9 & \underline{92.9} & 40.9
& 65.1 & 73.7 & 53.9 & 68.3 & 60.2 & 74.4 & 45.6 & 57.8 \\

RTQA (2025)~\citep{gong2025rtqa}
& 76.5 & 90.2 & 42.4
& 29.8 & 60.8 & 21.8 & 13.5
& 20.1 & 22.9 & 19.7 & 15.6 \\

TECQA (2026)~\citep{liu2026tecqa}
& 81.1 & 88.9 & 61.6
& 34.0 & 46.2 & 27.0 & 30.9
& 38.0 & 55.4 & 24.9 & 27.6 \\

\arrayrulecolor{black!35}\specialrule{0.4pt}{0pt}{0pt}\arrayrulecolor{black}

TempAgent (2025)~\citep{hu2025tempagent}
& 70.2 & 85.7 & 31.6
& 43.0 & 63.2 & 29.1 & 39.8
& 49.9 & 54.1 & 44.1 & 51.6 \\

% MultiTQ: ARI Table 1 (Simple/Complex map to Single/Multiple).
ARI (2024)~\citep{chen2024ari}
& 38.0 & 68.0 & 21.0
& 37.1 & 79.9 & 17.3 & 22.0
& 51.7 & 76.1 & 35.3 & 33.8 \\

MemoTime (2026)~\citep{tan2026memotime}
& 77.9 & 86.8 & 53.8
& 27.8 & 72.9 & 10.0 & 9.2
& 40.7 & 69.0 & 19.5 & 23.6 \\

\arrayrulecolor{black!35}\specialrule{0.4pt}{0pt}{0pt}\arrayrulecolor{black}

Temp-R1 (2026)~\citep{gong2026tempr1}
& 78.0 & 88.8 & 55.0
& \underline{70.5} & \textbf{96.0} & 48.6 & 67.2 & 64.2 & \textbf{86.6} & 38.8 & 59.5 \\

TempTool-R1 (2026)~\citep{huang2026temptoolr1}
& 85.1 & \textbf{95.2} & 75.3
& 51.7 & 80.9 & 42.6 & 49.5 & 47.5 & 76.1 & 40.2 & 45.3 \\

TKG-Thinker (2026)~\citep{jiang2026tkgthinker}
& 85.5 & 91.0 & 72.1
& 46.0 & 56.7 & 29.4 & 52.3 & 50.8 & 55.6 & 40.9 & 58.3 \\

\midrule

\rowcolor{qwenrow}

\system (Qwen3-4B)
& \underline{86.8} & 91.3 & \underline{75.9}
& \textbf{71.8} & 82.3 & \textbf{56.8} & \textbf{76.9} & \textbf{72.6} & \underline{80.3} & \textbf{62.4} & \textbf{75.1} \\

\rowcolor{llamarow}

\system (Llama-3.1-8B)
& \textbf{89.9} & 91.7 & \textbf{85.5}
& 69.2 & \underline{82.9} & \underline{55.1} & \underline{71.0} & \underline{70.5} & 79.6 & \underline{59.6} & \underline{71.4} \\
\bottomrule
\end{tabular}
\end{adjustbox}
\end{table*}

\noindent\textbf{Evidence Synthesis.}
Evidence synthesis integrates observations and reports from different workers to resolve the temporal constraints of the original question.
The worker messages from turn $k$ are added to the collaboration history $\mathcal{H}_{<k+1}$, which informs the coordinator's next response in Eq.~(\ref{eq:delegation}).
With the original timestamps preserved, the coordinator can jointly reason over facts from different subtasks, such as determining their temporal order for a before constraint or intersecting their intervals for an overlap constraint.
Thus, the workers act as the retriever and the coordinator acts as the reasoner in Eq.~(\ref{eq:tkgqa}), where the retrieved fact set is $\mathcal{C}^*=\bigcup_{k,j}\mathcal{O}_{k,j}$.
The coordinator outputs the final answer $a$ when sufficient evidence is available; otherwise, it delegates additional subtasks, for up to $K$ turns.

A complete collaboration produces a trajectory $\tau$ of coordinator and worker responses, whose probability factorizes across the two roles:
\begin{equation}
\begin{aligned}
P_\theta(\tau\mid q,\mathcal{G})={}&
\prod_{k}\pi_\theta(y^{\mathrm{c}}_k\mid p_{\mathrm{c}},q,\mathcal{H}_{<k})\\
&\cdot\prod_{k,j,u}\pi_\theta(y^{\mathrm{w}}_{k,j,u}\mid p_{\mathrm{w}},q,d_{k,j},\xi_{k,j,<u}).
\end{aligned}
\label{eq:trajectory}
\end{equation}
As tool observations are not generated by $\pi_\theta$, they appear only as conditions in Eq.~(\ref{eq:trajectory}).
This factorization shows that both coordinator decisions and worker investigations are governed by the same policy parameters $\theta$, which are initialized and optimized in Sections~\ref{sec:initialization} and~\ref{sec:selfplay}.

\subsection{Shared Policy Initialization}
\label{sec:initialization}

Shared policy initialization equips the shared policy with the ability to perform both roles and collaborate before self-play.
It consists of two steps: Demonstration Generation and Role-Conditioned Distillation.

\noindent\textbf{Demonstration Generation.}
Demonstration generation produces complete collaboration trajectories on the TKG to supervise the shared policy.
A base LLM does not reliably perform subtask delegation, tool invocation, and message exchange within this collaboration framework.
We therefore employ a teacher LLM that acts as both the coordinator and workers, using the same role-specific prompts and tools.
Each demonstration thus records both how workers generate messages and how the coordinator uses them.
Since the teacher may also produce incorrect answers, we retain only trajectories with correct final answers to construct the demonstration set $\mathcal{S}$.

\noindent\textbf{Role-Conditioned Distillation.}
Role-conditioned distillation trains a single shared policy to perform both roles using the collected demonstrations.
Each demonstration $\tau\in\mathcal{S}$ is divided into one coordinator sample and one sample per worker, each conditioned on its corresponding role prompt, yielding 2,000 training samples per dataset.
Coordinator samples supervise subtask delegation and evidence synthesis, while worker samples supervise timeline investigation, tool invocation, and message generation.
Both types of samples update the same policy parameters by maximizing the likelihood of the demonstrations:
\begin{equation}
\mathcal{J}_{\mathrm{distill}}(\theta)=\mathbb{E}_{\tau\sim\mathcal{S}}\big[\log P_\theta(\tau\mid q,\mathcal{G})\big],
\label{eq:sft}
\end{equation}
where $P_\theta(\tau\mid q,\mathcal{G})$ is the trajectory probability defined in Eq.~(\ref{eq:trajectory}).
Tool observations and returned messages serve only as context and are excluded from the loss.
The resulting policy, denoted $\pi_{\mathrm{ref}}$, serves as both the initialization and reference policy for self-play in Section~\ref{sec:selfplay}.
Although the initialized policy can collaborate with itself, both roles have only learned from teacher-generated interactions rather than their own collaborative experience.

\subsection{Co-Evolving Policy Optimization}
\label{sec:selfplay}

Co-evolving policy optimization jointly improves subtask delegation, timeline investigation, and evidence synthesis through self-play, where the shared policy learns from the final outcomes of its own collaborations.
It consists of three steps: Self-Play Rollout, Cross-Role Credit Assignment, and Leave-One-Out PPO.

\noindent\textbf{Self-Play Rollout.}
Self-play rollout generates complete collaborations in which the current shared policy plays both the coordinator and worker roles~\cite{eisenstein2025collaborative}.
Unlike teacher-generated demonstrations, the coordinator and workers now interact through subtasks and messages produced by the shared policy itself.
For each training question $q$, a snapshot $\pi_{\mathrm{old}}$ of the shared policy independently samples $N\geq2$ trajectories $\{\tau_n\}_{n=1}^{N}$ according to Eq.~(\ref{eq:trajectory}).
Without reference labels for intermediate decisions, each collaboration is evaluated by its final answer.
The shared policy maximizes the expected final-answer reward:
\begin{equation}
\mathcal{J}(\theta)=\mathbb{E}_{q}\,\mathbb{E}_{\tau\sim P_\theta(\cdot\mid q,\mathcal{G})}\big[R(\tau)\big],
\label{eq:selfplay}
\end{equation}
where $R(\tau)=1$ if the final answer of $\tau$ belongs to the answer set $\mathcal{A}_q$, and $R(\tau)=0$ otherwise, including when no answer is returned.

\noindent\textbf{Cross-Role Credit Assignment.}
Cross-role credit assignment propagates the outcome of each collaboration to the decisions made by both roles.
By the likelihood-ratio identity, the gradient of Eq.~(\ref{eq:selfplay}) decomposes into contributions from the coordinator and workers:
\begin{equation}
\nabla_\theta\mathcal{J}(\theta)=\mathbb{E}_{q,\tau}\big[R(\tau)\,\big(g^{\mathrm{c}}(\tau)+g^{\mathrm{w}}(\tau)\big)\big],
\label{eq:gradient}
\end{equation}
where $g^{\mathrm{c}}(\tau)$ and $g^{\mathrm{w}}(\tau)$ sum $\nabla_\theta\log\pi_\theta$ over the coordinator and worker responses in Eq.~(\ref{eq:trajectory}), respectively.
Thus, the final-answer reward provides a shared learning signal for both roles, without requiring separate supervision for their intermediate decisions.
Each role is reinforced only when it succeeds together with its current partner.
As this partner keeps improving through the same updates, the coordinator and workers co-evolve~\cite{wang2025cure}.

\noindent\textbf{Leave-One-Out PPO.}
Leave-one-out PPO (LOOP) estimates trajectory-level advantages from final-answer rewards and uses them to optimize the shared policy.
LOOP was developed for long-horizon interactive agents that receive only a task-completion reward~\cite{chen2025loop}, which matches our setting where each collaboration receives only a final-answer reward.
Instead of training a separate value model, we estimate the advantage of each trajectory $\tau_n$ by comparing its reward with those of the other trajectories sampled for the same question~\cite{kool2019buy4,ahmadian2024rloo}:
\begin{equation}
\hat{A}_n=R(\tau_n)-\frac{1}{N-1}\sum_{n'\neq n}R(\tau_{n'}).
\label{eq:advantage}
\end{equation}
Excluding $\tau_n$ from its own baseline preserves the unbiasedness of the on-policy policy-gradient estimator~\cite{kool2019buy4}.
The resulting advantage $\hat{A}_n$ is shared across all tokens generated by the coordinator and workers in $\tau_n$.
We optimize the shared policy using the clipped PPO objective~\cite{schulman2017ppo}, with dual clipping~\cite{ye2020dualclip} and KL regularization toward the reference policy $\pi_{\mathrm{ref}}$ from Section~\ref{sec:initialization}:
\begin{equation}
\begin{aligned}
\mathcal{J}_{\mathrm{LOOP}}(\theta)
={}&\hat{\mathbb{E}}_{n,\ell}
\big[\mathcal{L}^{\mathrm{DC}}_{n,\ell}\big]\\
&-\beta\,\hat{\mathbb{E}}_{n,\ell}
\big[D_{\mathrm{KL}}(\pi_\theta(\cdot\mid h_{n,\ell})
\,\|\,\pi_{\mathrm{ref}}(\cdot\mid h_{n,\ell}))\big],
\end{aligned}
\label{eq:rl}
\end{equation}
where $\beta$ weights the KL term and $\mathcal{L}^{\mathrm{DC}}_{n,\ell}$ is the dual-clipped PPO surrogate:
\begin{equation}
\begin{aligned}
\mathcal{L}^{\mathrm{clip}}_{n,\ell}&=\min\big(\kappa_{n,\ell}\hat{A}_n,\ \mathrm{clip}(\kappa_{n,\ell},1-\epsilon,1+\epsilon)\hat{A}_n\big),\\
\mathcal{L}^{\mathrm{DC}}_{n,\ell}&=
\begin{cases}
\max(\mathcal{L}^{\mathrm{clip}}_{n,\ell},\eta\hat{A}_n), & \hat{A}_n<0,\\
\mathcal{L}^{\mathrm{clip}}_{n,\ell}, & \hat{A}_n\geq0.
\end{cases}
\end{aligned}
\label{eq:dualclip}
\end{equation}
Here $\kappa_{n,\ell}=\pi_\theta(x_{n,\ell}\mid h_{n,\ell})/\pi_{\mathrm{old}}(x_{n,\ell}\mid h_{n,\ell})$ is the probability ratio of the $\ell$-th generated token $x_{n,\ell}$ of $\tau_n$ under its role-specific context $h_{n,\ell}$, $\epsilon$ controls PPO clipping, and $\eta>1$ controls dual clipping.
The empirical expectation $\hat{\mathbb{E}}_{n,\ell}$ averages over tokens generated by both roles, excluding tool observations and returned messages.
PPO clipping limits policy updates relative to $\pi_{\mathrm{old}}$, while dual clipping bounds the objective for negative advantages with large ratios.
The KL term regularizes updates toward the collaboration behavior learned from demonstrations.

After each update, the improved shared policy generates new collaborations in both roles, allowing the coordinator and workers to adapt to each other's evolving behavior.
At inference, the trained policy performs all roles in the collaboration described in Section~\ref{sec:collaboration}, requiring only a single model.

\section{Experiments}
\label{sec:experiments}

In this section, we evaluate \system by answering the following research questions:
\textbf{RQ1}. Can \system outperform existing TKGQA methods and closed-source LLMs?
\textbf{RQ2}. How do timeline reading and multi-agent collaboration contribute to \system?
\textbf{RQ3}. How do tool observations and reports in worker messages affect evidence synthesis?
\textbf{RQ4}. How does co-evolving self-play improve the coordinator and workers?
\textbf{RQ5}. Do the gains of \system come from additional computation?

\subsection{Experimental Settings}
\label{sec:experimental_setup}
\paragraph{Datasets}
We evaluate \system on three TKGQA benchmarks: MultiTQ~\citep{chen2023multitq}, TimelineKGQA-CronQuestion, and TimelineKGQA-ICEWS-Actor~\citep{sun2025timelinekgqa}.
MultiTQ questions cover day, month, and year granularities and are categorized as \texttt{Single} or \texttt{Multiple} according to the number of temporal constraints.
Both TimelineKGQA benchmarks group questions into \texttt{Simple}, \texttt{Medium}, and \texttt{Complex} categories, requiring one, two, and three supporting facts, respectively.
Further dataset descriptions and statistics are provided in Appendix~\ref{app:datasets}.

\paragraph{Baselines}
We compare \system with three groups of LLM-based TKGQA methods:
(1) retrieval-based methods, including TimeR$^4$~\citep{qian2024timer4}, PoK~\citep{qian2025pok}, RTQA~\citep{gong2025rtqa}, and TECQA~\citep{liu2026tecqa};
(2) prompting-based agents, including TempAgent~\citep{hu2025tempagent}, ARI~\citep{chen2024ari}, and MemoTime~\citep{tan2026memotime};
and (3) RL-based agents, including Temp-R1~\citep{gong2026tempr1}, TempTool-R1~\citep{huang2026temptoolr1}, and TKG-Thinker~\citep{jiang2026tkgthinker}.
We also compare \system with two closed-source LLMs, GPT-5.6 Sol and GPT-6 Astra.
Brief descriptions of the baselines are provided in Appendix~\ref{app:baselines}.

\paragraph{Implementation Details}
Following prior work~\citep{gong2026tempr1,huang2026temptoolr1,jiang2026tkgthinker}, we adopt Hits@1 as the evaluation metric.
We fine-tune Qwen3-4B-Instruct~\citep{qwen3technicalreport} and Llama-3.1-8B-Instruct~\citep{grattafiori2024llama3} as backbones, using the same E5 retriever as~\citet{jin2025searchr1} to retrieve the top 15 facts per query.
Role-conditioned distillation uses demonstrations generated by Qwen3-30B-A3B-Instruct~\citep{qwen3technicalreport}.
For each backbone, we train one model on MultiTQ and a single model jointly on both TimelineKGQA benchmarks.
Table~\ref{tab:main_results} reports results on the full test sets.
The other experiments use 1,000 questions sampled from the dev set of each benchmark.
Detailed training settings are provided in Appendix~\ref{app:implementation_details}.

\subsection{Main Results (RQ1)}
\label{sec:main_results}
Table~\ref{tab:main_results} reports the performance of \system and the baselines on the three benchmarks.
\system achieves the best overall Hits@1 on all three benchmarks.
The Qwen variant outperforms all baselines on every benchmark, while the Llama variant achieves the highest accuracy of 89.9\% on MultiTQ.
Among the baselines, RL-based agents perform best, with TKG-Thinker leading on MultiTQ and Temp-R1 leading on both TimelineKGQA benchmarks.
This is because RL trains the underlying LLMs with interaction feedback.
Although these agents are also trained with RL, \system still outperforms them by further learning subtask delegation, timeline investigation, and evidence synthesis through co-evolving self-play.

The advantage of \system is most evident on questions with multiple supporting facts.
Both variants outperform all baselines on MultiTQ \texttt{Multiple} and on the \texttt{Medium} and \texttt{Complex} categories of both TimelineKGQA benchmarks.
For example, on TimelineKGQA-ICEWS-Actor \texttt{Complex}, the Qwen variant reaches 75.1\%, compared with 59.5\% for Temp-R1.
This advantage is attributed to the timeline-centric collaboration of \system.
Workers investigate the timelines of different subtasks in separate contexts, while the coordinator synthesizes their observations with original timestamps for joint temporal reasoning.
On \texttt{Single} and \texttt{Simple} questions, TempTool-R1 and Temp-R1 achieve higher accuracy, as these questions can be answered within a single investigation.

\begin{table}[t]
\centering
\captionsetup{skip=4pt}
\caption{Hits@1 (\%) of \system and closed-source LLMs. Bold denotes the best score in each column.}
\label{tab:gpt_comparison}

\small
\renewcommand{\arraystretch}{1.0}
\setlength{\tabcolsep}{4.5pt}
\begin{adjustbox}{max width=\columnwidth}
\begin{tabular}{lccc}
\toprule
\multirow{2}{*}{Method} & \multirow{2}{*}{MultiTQ}
    & \multicolumn{2}{c}{TimelineKGQA} \\
\cmidrule(lr){3-4}
& & CronQuestion & ICEWS-Actor \\
\midrule
GPT-5.6 Sol & 82.0 & 65.4 & 63.6 \\
GPT-6 Astra & 82.2 & 69.0 & 64.4 \\
\midrule
\rowcolor{qwenrow}
\system (Qwen3-4B) & 87.4 & \textbf{71.1} & \textbf{71.7} \\
\rowcolor{llamarow}
\system (Llama-3.1-8B) & \textbf{89.5} & 70.4 & 69.9 \\
\bottomrule
\end{tabular}
\end{adjustbox}
\Description{Hits@1 of GPT-5.6 Sol, GPT-6 Astra, and both CoSTA backbones on three benchmarks. Both CoSTA backbones score higher than the GPT models on all three benchmarks.}
\end{table}

Table~\ref{tab:gpt_comparison} further compares \system with two closed-source LLMs, GPT-5.6 Sol and GPT-6 Astra, both equipped with the E5 retriever used in our implementation.
Despite access to E5-retrieved facts, both the 4B and 8B variants of \system outperform these models on all three benchmarks.
For example, the Llama variant achieves 89.5\% on MultiTQ, compared with 82.2\% for GPT-6 Astra.
These results highlight the effectiveness of timeline-centric collaboration and co-evolving self-play in enabling compact open-source backbones to surpass closed-source LLMs in TKGQA.

\subsection{Ablation Experiments (RQ2)}
\label{sec:ablation}

\begin{table}[t]
\centering
\captionsetup{skip=4pt}
\caption{Ablation of timeline reading and collaboration (Hits@1, \%). TimelineKGQA reports the average over its two benchmarks. Bold denotes the best score in each column.}
\label{tab:ablation}

\small
\renewcommand{\arraystretch}{1.0}
\setlength{\tabcolsep}{4.5pt}
\begin{adjustbox}{max width=\columnwidth}
\begin{tabular}{lcccc}
\toprule
\multirow{2}{*}{Method} & \multicolumn{2}{c}{\cellcolor{qwenrow}\textbf{Qwen3-4B}} & \multicolumn{2}{c}{\cellcolor{llamarow}\textbf{Llama-3.1-8B}} \\
\cmidrule(lr){2-3} \cmidrule(lr){4-5}
& MultiTQ & TimelineKGQA & MultiTQ & TimelineKGQA \\
\midrule
\system & \textbf{87.4} & \textbf{71.4} & \textbf{89.5} & \textbf{70.2} \\
\quad w/o timeline reading & 69.7 & 71.2 & 64.8 & 69.0 \\
\quad w/o collaboration & 82.9 & 67.1 & 89.1 & 65.4 \\
\quad w/o both & 65.6 & 66.0 & 62.7 & 64.3 \\
\bottomrule
\end{tabular}
\end{adjustbox}
\end{table}

\begin{figure}[t]
  \centering
  \includegraphics[width=\columnwidth]{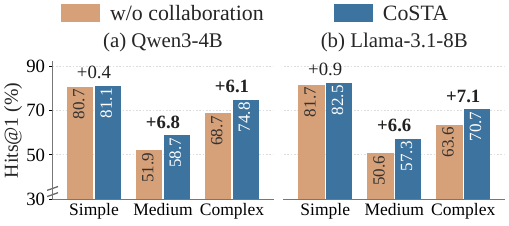}
  \caption{Hits@1 (\%) on TimelineKGQA questions of different difficulty.}
  \Description{Paired bars for Qwen3-4B and Llama-3.1-8B. Collaboration improves accuracy by less than one point on Simple questions and by 6.1 to 7.1 points on Medium and Complex questions.}
  \label{fig:difficulty}
\end{figure}

Table~\ref{tab:ablation} reports the ablation of two core designs of \system, namely timeline reading and multi-agent collaboration.
To isolate their contributions, all variants share the trained policy of \system and differ only at inference.
Without timeline reading, workers can only access facts through semantic retrieval.
Without collaboration, a single agent investigates and answers each question within one context.

Timeline reading is essential on MultiTQ.
Removing it reduces Hits@1 from 87.4\% to 69.7\% for Qwen and from 89.5\% to 64.8\% for Llama.
This is because many MultiTQ questions ask for the first or last fact, or the fact before or after another fact, which requires reading neighboring facts on a timeline.

Collaboration mainly improves TimelineKGQA.
Without collaboration, the TimelineKGQA accuracy drops from 71.4\% to 67.1\% for Qwen and from 70.2\% to 65.4\% for Llama.
As shown in Figure~\ref{fig:difficulty}, the drop concentrates on questions with multiple supporting facts.
Removing collaboration lowers the accuracy on \texttt{Medium} and \texttt{Complex} questions by 6.1 to 7.1 points for both backbones, compared with less than one point on \texttt{Simple} questions.
\texttt{Medium} and \texttt{Complex} questions require joint reasoning over the time intervals of two or three supporting facts.
In \system, workers investigate these facts in separate contexts, while the coordinator synthesizes their observations with original timestamps.

Overall, timeline reading mainly benefits MultiTQ, whereas collaboration mainly benefits TimelineKGQA.
Removing both designs yields the lowest accuracy in all settings, showing that the two designs address complementary challenges.

\subsection{Impact of Worker Messages (RQ3)}
\label{sec:messages}

\begin{figure}[t]
  \centering
  \includegraphics[width=\columnwidth]{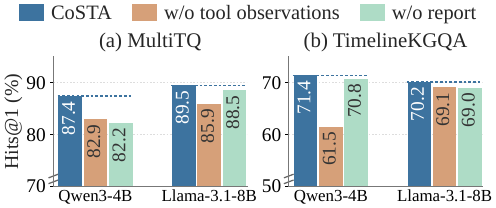}
  \caption{Hits@1 (\%) with different contents of worker messages. TimelineKGQA reports the average over its two benchmarks. Dashed lines mark \system.}
  \Description{Grouped bars for Qwen3-4B and Llama-3.1-8B on MultiTQ and TimelineKGQA. Removing tool observations or reports lowers accuracy in all settings; removing tool observations lowers Qwen's TimelineKGQA accuracy from 71.4 to 61.5.}
  \label{fig:returns}
\end{figure}

Figure~\ref{fig:returns} shows the effect of removing tool observations or reports from worker messages, as defined in Section~\ref{sec:collaboration}.
All variants use the same trained policy of \system and differ only in what workers return at inference time.
Without tool observations, workers return only reports; without reports, they return only tool observations.

Tool observations play a central role in evidence synthesis.
Removing them reduces accuracy in all four settings, with the largest drop for Qwen on TimelineKGQA (from 71.4\% to 61.5\%).
While reports summarize worker findings in natural language, tool observations retain the original timestamps needed to order facts and intersect their time intervals across subtasks.

Reports also contribute to performance.
Removing them reduces accuracy in all four settings, with the largest drop for Qwen on MultiTQ (from 87.4\% to 82.2\%).
On MultiTQ, timeline reading returns many neighboring facts for each subtask.
Reports relate these observations to the subtask, helping the coordinator identify relevant evidence.

Together, tool observations provide the original timestamps, whereas reports help the coordinator identify relevant evidence.
Retaining both yields the highest accuracy in all settings.

\subsection{Analysis of Co-Evolution (RQ4)}
\label{sec:selfplay_analysis}

\begin{figure}[t]
  \centering
  \includegraphics[width=\columnwidth]{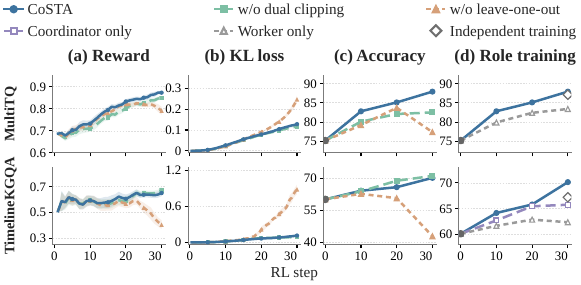}
  \caption{Training dynamics of \system with Llama-3.1-8B. Reward and KL loss are trailing means with standard-error bands; accuracy is Hits@1 (\%). TimelineKGQA reports the average over its two benchmarks.}
  \Description{A two-by-four grid of line plots for MultiTQ and TimelineKGQA. Without the leave-one-out baseline, the reward drops, the KL loss rises sharply, and TimelineKGQA accuracy falls to 42.9 percent. Training both roles together reaches 87.9 and 70.2 percent, above training only the coordinator, only the workers, or pairing independently trained roles.}
  \label{fig:selfplay}
\end{figure}

Figure~\ref{fig:selfplay} analyzes co-evolving policy optimization with Llama, where all variants start from the same distilled policy.
The variant w/o leave-one-out replaces the leave-one-out baseline in Eq.~(\ref{eq:advantage}) with the mean reward of the training batch.
The variant w/o dual clipping removes dual clipping from Eq.~(\ref{eq:dualclip}).
\textit{Coordinator only} and \textit{worker only} train one role and freeze the other at the distilled policy.
\textit{Independent training} pairs the coordinator from \textit{coordinator only} with the workers from \textit{worker only}.

Leave-one-out PPO stabilizes self-play.
Without the leave-one-out baseline, the reward declines and the KL loss increases rapidly after step 20.
The TimelineKGQA accuracy drops from 62.7\% at step 10 to 42.9\% at step 30, compared with 70.2\% for \system.
Removing dual clipping also lowers the MultiTQ accuracy from 87.9\% to 82.5\%.
This is because a batch-mean baseline mixes questions of different difficulty, so the advantage of a trajectory mainly reflects the difficulty of its question.
The leave-one-out baseline compares trajectories sampled for the same question, which provides a stable learning signal for both roles.

Training both roles together yields the highest accuracy.
On TimelineKGQA, \system reaches 70.2\%, compared with 65.8\% for \textit{coordinator only} and 62.4\% for \textit{worker only}.
\textit{Independent training} reaches 67.2\%, which remains 3.0 points below \system.
This is because the coordinator synthesizes evidence from worker messages, whereas workers investigate the subtasks delegated by the coordinator.
When only one role is trained, it adapts to a fixed partner, which limits its improvement after step 20.
When the two roles are trained independently, each adapts to a partner different from the one it collaborates with at inference.
With cross-role credit assignment, each role is reinforced only when it succeeds together with its current partner, so both roles improve together.

\begin{figure}[t]
  \centering
  \includegraphics[width=\columnwidth]{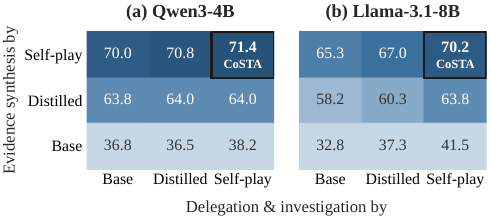}
  \caption{Hits@1 (\%) on TimelineKGQA when the subtasks and worker messages of each training stage (columns) are synthesized by the coordinator of each stage (rows).}
  \Description{Two three-by-three heatmaps for Qwen3-4B and Llama-3.1-8B. Rows give the stage of evidence synthesis and columns the stage of delegation and investigation. Accuracy rises along both axes and peaks when both come from the self-play policy, which is CoSTA, at 71.4 and 70.2 percent.}
  \label{fig:synthesis}
\end{figure}

Both roles improve through co-evolving policy optimization.
Figure~\ref{fig:synthesis} pairs the subtasks and worker messages produced at each training stage with the coordinator at each stage for evidence synthesis.
With the coordinator of the self-play policy, replacing the delegation and investigation of the base policy with those of the self-play policy improves the accuracy from 70.0\% to 71.4\% for Qwen and from 65.3\% to 70.2\% for Llama.
With the delegation and investigation of the self-play policy, replacing the distilled coordinator with the self-play coordinator improves the accuracy by 7.4 and 6.4 points, respectively.
Using the self-play policy for both roles, which is \system itself, yields the highest accuracy for both backbones.
This is because cross-role credit assignment propagates the final outcome to the decisions of both roles, so self-play improves subtask delegation and timeline investigation in addition to evidence synthesis.

Overall, leave-one-out PPO provides a stable learning signal for self-play, whereas cross-role credit assignment enables the coordinator and workers to co-evolve.

\subsection{Efficiency Analysis (RQ5)}
\label{sec:efficiency}

\begin{figure}[t]
  \centering
  \includegraphics[width=\columnwidth]{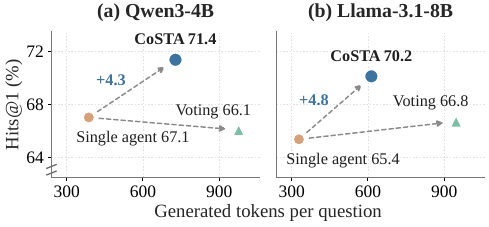}
  \caption{Hits@1 (\%) and generated tokens per question on TimelineKGQA. All variants use the same trained policy; arrows start from the single agent.}
  \Description{Two scatter plots for Qwen3-4B and Llama-3.1-8B. CoSTA reaches 71.4 and 70.2 percent with 728 and 614 generated tokens per question. Majority voting over three single-agent runs uses 977 and 947 tokens but reaches only 66.1 and 66.8 percent. The single agent uses 387 and 328 tokens and reaches 67.1 and 65.4 percent.}
  \label{fig:efficiency}
\end{figure}

Figure~\ref{fig:efficiency} compares \system on TimelineKGQA with the single agent in Section~\ref{sec:ablation} and with \textit{voting}, which runs the single agent three times independently and takes the majority answer.
The gains of \system cannot be explained solely by
generating more tokens.
\textit{Voting} generates more tokens than \system for both backbones, yet its accuracy is 5.3 and 3.4 points lower.
This is because voting spends additional tokens on repeated attempts at the whole question, whereas \system spends them on different subtasks investigated in separate contexts.

\section{Conclusion}
\label{sec:conclusion}

In this paper, we propose Co-Evolving Self-Play Timeline Agents (\system), a timeline-centric collaborative framework that learns collaboration among agents for complex TKGQA.
Specifically, \system organizes TKG facts into chronologically ordered timelines and delegates the investigations of different event histories to workers in separate contexts.
A coordinator then synthesizes their original timestamped observations for joint temporal reasoning.
To learn this collaboration, \system initializes a policy shared by both roles via role-conditioned distillation and jointly optimizes the coordinator and workers through co-evolving self-play with final-answer rewards.
Experimental results on three TKGQA benchmarks demonstrate that \system achieves state-of-the-art overall performance, with particularly strong gains on questions involving multiple supporting facts.
Further analyses show that separate investigations by workers and joint optimization of both roles each contribute to these gains.

\paragraph{Limitations.}
CoSTA relies on binary final-answer rewards for policy optimization. Although this enables end-to-end learning without intermediate supervision, it provides no role- or step-specific feedback, which may limit sample efficiency and scalability to longer collaborations.

\bibliographystyle{ACM-Reference-Format}
\balance
\bibliography{references}

% \nobalance
\appendix
\section*{Appendix}
\section{Dataset Details}
\label{app:datasets}

MultiTQ distinguishes questions by the number of temporal constraints, whereas TimelineKGQA organizes them by the number of supporting facts and associated temporal operations.
Tables~\ref{tab:multitq_statistics} and~\ref{tab:timelinekgqa_statistics} report question counts by category and split for MultiTQ and the two TimelineKGQA benchmarks, respectively.

MultiTQ~\citep{chen2023multitq} is generated from ICEWS05--15~\citep{garcia2018learning}, where events are recorded at day-level granularity.
Its questions involve temporal constraints at day, month, and year granularities.
The benchmark excludes questions that can be answered without using their temporal constraints.

\emph{Single} questions involve one temporal constraint and fall into three categories: \texttt{Equal}, \texttt{Before/After}, and \texttt{First/Last}.
\emph{Multiple} questions involve combinations of temporal constraints and fall into three categories: \texttt{Equal Multi}, \texttt{After First}, and \texttt{Before Last}.
For example, ``Who last visited South Korea in 2014?'' combines an \texttt{Equal} constraint with \texttt{First/Last} selection.
Each question is paired with a set of entity or time answers.

TimelineKGQA~\citep{sun2025timelinekgqa} generates two benchmarks using the CronQuestions~\citep{saxena2021cronkgqa} and ICEWS Actor~\citep{boschee2015icews} knowledge graphs.
We refer to these benchmarks as TimelineKGQA-CronQuestion and TimelineKGQA-ICEWS-Actor, respectively.

Questions are grouped into Simple, Medium, and Complex categories, requiring one, two, and three supporting facts, respectively.
Simple questions query entities, start and end times, or durations associated with a single fact.
Medium questions involve two facts and require reasoning about Allen interval relations, computing interval unions or intersections, or calculating and comparing durations.
The Complex category extends these operations to three facts and includes temporal ranking questions.
Answers include entities, relations, temporal values (timestamps, intervals, and durations), yes/no responses, and rankings.

\begin{table}[h]
\centering
\caption{Question counts by category and split in MultiTQ, as reported by~\citet{chen2023multitq}.}
\label{tab:multitq_statistics}

\small
\renewcommand{\arraystretch}{1.12}
\setlength{\tabcolsep}{4.5pt}

\begin{tabular}{llrrr}
\toprule
Group & Question type & Train & Validation & Test \\
\midrule

\multirow{3}{*}{Single}
& Equal & 135,890 & 18,983 & 17,311 \\
& Before/After & 75,340 & 11,655 & 11,073 \\
& First/Last & 72,252 & 11,097 & 10,480 \\

\cmidrule(lr){2-5}

\multirow{3}{*}{Multiple}
& Equal Multi & 16,893 & 3,213 & 3,207 \\
& After First & 43,305 & 6,499 & 6,266 \\
& Before Last & 43,107 & 6,532 & 6,247 \\

\midrule
\multicolumn{2}{l}{\textbf{Total}}
& \textbf{386,787}
& \textbf{57,979}
& \textbf{54,584} \\

\bottomrule
\end{tabular}
\end{table}

\begin{table}[h]
\centering
\caption{Question counts by difficulty and split for the two TimelineKGQA benchmarks, as reported by~\citet{sun2025timelinekgqa}.}
\label{tab:timelinekgqa_statistics}

\small
\renewcommand{\arraystretch}{1.12}
\setlength{\tabcolsep}{3.5pt}

\begin{tabular}{@{}llrrr@{}}
\toprule
Dataset & Difficulty & Train & Validation & Test \\
\midrule

\multirow{4}{*}{\shortstack[l]{TimelineKGQA-\\CronQuestion}}
& Simple  & 7,200 & 2,400 & 2,400 \\
& Medium  & 8,252 & 2,751 & 2,751 \\
& Complex & 9,580 & 3,193 & 3,193 \\
\cmidrule(lr){2-5}
& \textbf{Total}
& \textbf{25,032}
& \textbf{8,344}
& \textbf{8,344} \\

\midrule

\multirow{4}{*}{\shortstack[l]{TimelineKGQA-\\ICEWS-Actor}}
& Simple  & 17,982 & 5,994 & 5,994 \\
& Medium  & 15,990 & 5,330 & 5,330 \\
& Complex & 19,652 & 6,550 & 6,550 \\
\cmidrule(lr){2-5}
& \textbf{Total}
& \textbf{53,624}
& \textbf{17,874}
& \textbf{17,874} \\

\bottomrule
\end{tabular}

\end{table}

\section{Baseline Methods}
\label{app:baselines}

\noindent\textbf{ARI}~\citep{chen2024ari} guides an LLM agent with reasoning strategies induced from previous question-answering trajectories.
It groups similar questions and summarizes successful and unsuccessful trajectories into guidance for subsequent temporal tool use.

\noindent\textbf{TimeR$^4$}~\citep{qian2024timer4} uses retrieved facts to rewrite implicit temporal constraints as explicit timestamps, then retrieves and reranks evidence by semantic and temporal relevance.
Both the temporal retriever and the answering LLM are fine-tuned.

\noindent\textbf{TempAgent}~\citep{hu2025tempagent} follows a ReAct~\citep{yao2023react} loop that alternates reasoning with semantic and temporally constrained search.
Its retrieval tools support on, before and after conditions at day, month and year granularities.
The agent uses retrieved observations to guide subsequent actions without task-specific LLM fine-tuning.

\noindent\textbf{RTQA}~\citep{gong2025rtqa} recursively decomposes temporal questions into subquestions organized as a tree.
At each node, an LLM selects an answer using both retrieved facts and answers from child subquestions.
The framework uses prompting without task-specific LLM fine-tuning.

\noindent\textbf{PoK}~\citep{qian2025pok} generates a plan of retrieval, temporal ranking, and reasoning steps, using intermediate answers to resolve subsequent temporal constraints.
It uses a fine-tuned temporal retriever to obtain relevant facts and a fine-tuned LLM to generate answers from those facts.

\noindent\textbf{MemoTime}~\citep{tan2026memotime} decomposes temporal questions into a Tree of Time, using past reasoning experiences to guide decomposition and tool selection.
It combines graph traversal with dense retrieval and checks whether the evidence is sufficient to answer each subquestion.
Verified reasoning traces are stored in an external memory for reuse.

\noindent\textbf{TECQA}~\citep{liu2026tecqa} retrieves a subgraph using entity and relation constraints extracted from the question.
It selects events near the question's temporal anchors and orders them by temporal proximity to form an evidence chain.
A frozen LLM generates the answer from this chain.

\noindent\textbf{Temp-R1}~\citep{gong2026tempr1} combines external search with internal actions for planning, fact filtering, and temporal ranking.
It is trained with supervised fine-tuning followed by Group Relative Policy Optimization (GRPO)~\citep{shao2024deepseekmath}, using final-answer correctness as the reward.
Its reverse curriculum starts with complex questions and later adds simpler questions.

\noindent\textbf{TempTool-R1}~\citep{huang2026temptoolr1} interleaves reasoning with temporal tools for retrieving event times and entities under time constraints.
After supervised fine-tuning on tool-use trajectories, GRPO rewards output format and agreement with reference tool calls, including tool names, arguments and call order.

\noindent\textbf{TKG-Thinker}~\citep{jiang2026tkgthinker} plans and performs temporal searches over multiple turns, using retrieved observations to guide subsequent queries.
Supervised fine-tuning on reasoning trajectories is followed by online reinforcement learning.
Rewards account for final-answer correctness, interaction format and retrieval of facts containing the reference answer.

\section{Implementation Details}
\label{app:implementation_details}

\noindent\textbf{Role-Conditioned Distillation.}
We fine-tune Qwen3-4B-Instruct and Llama-3.1-8B-Instruct using demonstrations generated by Qwen3-30B-A3B-Instruct for questions in the training set.
Full-parameter fine-tuning runs for one epoch using AdamW~\citep{loshchilov2019adamw} with a batch size of 32, a learning rate of $2\times10^{-6}$, cosine decay, and a warmup ratio of 0.03.
The loss is computed only on model-generated tokens, with input prompts, tool observations, and returned messages masked out.
Training runs on four NVIDIA A800 GPUs (80GB) with FSDP and BF16 mixed precision.

\noindent\textbf{Reinforcement Learning.}
Following distillation, we train the shared policy with leave-one-out PPO~\citep{chen2025loop} in VeRL~\citep{sheng2025hybridflow}.
Each rollout batch contains 64 questions with $N=3$ trajectories per question for MultiTQ, or 24 questions with $N=8$ for TimelineKGQA.
Each batch is split by question into two mini-batches, giving two optimizer updates in one pass over the batch, with all trajectories of a question kept in the same mini-batch.
The rollout temperature is set to 1.0.
We set the clipping range $\epsilon = 0.2$, the dual-clip parameter $\eta = 3$, and the KL coefficient $\beta = 0.001$.
The distilled policy is frozen as the reference policy $\pi_{\mathrm{ref}}$.
We use AdamW with a learning rate of $1.5\times10^{-6}$, weight decay 0.01, and gradient clipping at 1.0.

\section{Theoretical Analysis}
\label{app:theory}

We analyze the role-specific gradients in Eq.~(\ref{eq:gradient}) and the leave-one-out advantage in Eq.~(\ref{eq:advantage}).

\noindent\textbf{Cross-Role Credit Assignment.}
Let $\mathcal{J}(\theta_{\mathrm{c}},\theta_{\mathrm{w}})$ denote Eq.~(\ref{eq:selfplay}) after assigning separate parameters to coordinator and worker factors in Eq.~(\ref{eq:trajectory}).

\begin{proposition}
\label{prop:cross_role}
At $\theta_{\mathrm{c}}=\theta_{\mathrm{w}}=\theta$, for $r\in\{\mathrm{c},\mathrm{w}\}$,
\[
\begin{aligned}
&\mathbb{E}_{q,\tau}[R(\tau)g^r(\tau)]\\
&\quad = \left.\nabla_{\theta_r}\mathcal{J}\right|_{\theta_{\mathrm{c}}=\theta_{\mathrm{w}}=\theta}.
\end{aligned}
\]
\end{proposition}

\begin{proof}
Conditioned on tool calls and the fixed TKG, tool observations have parameter-independent distributions.
Thus $\nabla_{\theta_r}\log P=g^r$.
The reward depends only on the completed trajectory, so the likelihood-ratio identity proves the claim; the chain rule recovers Eq.~(\ref{eq:gradient}).
\end{proof}

Thus the final-answer reward supplies an on-policy gradient signal to each role under the current partner policy.
In the single-role and independent-training ablations, each optimized role instead interacts with a partner fixed at $\pi_{\mathrm{ref}}$.
This decomposition describes the expected objective gradient; practical updates use the clipped PPO surrogate.

\noindent\textbf{Leave-One-Out Advantage.}
For a question $q$, let $\rho_q=\Pr_{\pi_{\mathrm{old}}}(R=1\mid q)$.
Its $N\geq2$ sampled rewards $R_n=R(\tau_n)$ are conditionally independent Bernoulli variables.
Let $s_n=\nabla_\theta\log P_\theta(\tau_n\mid q,\mathcal{G})$ at $\pi_\theta=\pi_{\mathrm{old}}$.

\begin{proposition}
\label{prop:loo}
At $\pi_\theta=\pi_{\mathrm{old}}$,
\[
\begin{aligned}
\mathbb{E}[\hat{A}_n s_n\mid q] &= \mathbb{E}[R_n s_n\mid q],\\
\mathbb{E}[\hat{A}_n^2\mid q] &= \frac{N}{N-1}\rho_q(1-\rho_q).
\end{aligned}
\]
For a batch of $M$ questions and $N$ trajectories each, let $R_{i,n}$ denote the reward and $\rho_i=\rho_{q_i}$.
Define $B=MN$, $b=B^{-1}\sum_{i,n}R_{i,n}$, $\bar{\rho}=M^{-1}\sum_i\rho_i$, $\bar{v}=M^{-1}\sum_i\rho_i(1-\rho_i)$, and $V_\rho=M^{-1}\sum_i(\rho_i-\bar{\rho})^2$.
Then
\[
\begin{aligned}
&\mathbb{E}\!\left[\frac{1}{B}\sum_{i,n}(R_{i,n}-b)^2\mid q_{1:M}\right]\\
&\qquad = \left(1-\frac{1}{B}\right)\bar{v}+V_\rho.
\end{aligned}
\]
\end{proposition}

\begin{proof}
The leave-one-out baseline excludes $\tau_n$ and is conditionally independent of its score.
Since $\mathbb{E}[s_n\mid q]=0$, the first identity follows~\citep{kool2019buy4}.
The advantage has zero conditional mean, and independence gives its variance as $\rho_q(1-\rho_q)[1+1/(N-1)]$.
For binary rewards, $B^{-1}\sum_{i,n}(R_{i,n}-b)^2=b-b^2$.
Conditional on the questions, $\mathbb{E}[b]=\bar{\rho}$ and $\mathrm{Var}(b)=\bar{v}/B$.
Substituting $\bar{\rho}(1-\bar{\rho})=\bar{v}+V_\rho$ proves the last identity.
\end{proof}

Averaging the leave-one-out second moment gives $N\bar{v}/(N-1)$, smaller than the batch-mean counterpart when $V_\rho>(1/(N-1)+1/B)\bar{v}$.
The between-question term $V_\rho$ reflects differences in question difficulty and is absent under within-question centering.
If all trajectories for a question receive the same reward, its leave-one-out advantages vanish.
This comparison concerns advantage second moments and does not guarantee lower variance for clipped PPO gradients.

\section{Additional Experiments}
\label{app:additional}

\subsection{Contribution of Training Stages}
\label{app:stages}

\begin{figure}[t]
  \centering
  \includegraphics[width=\columnwidth]{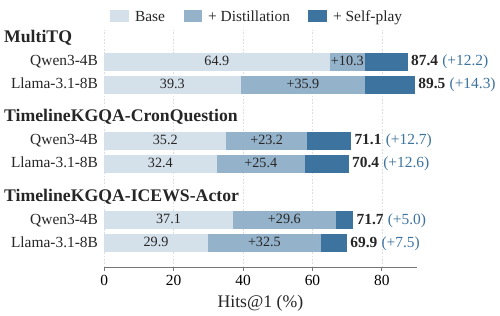}
  \caption{Hits@1 (\%) of the base, distilled, and self-play policies. Values in parentheses are the gains of co-evolving policy optimization.}
  \Description{Stacked horizontal bars for Qwen3-4B and Llama-3.1-8B on MultiTQ, TimelineKGQA-CronQuestion, and TimelineKGQA-ICEWS-Actor. Distillation raises accuracy by 10.3 to 35.9 points over the base policy. Self-play adds a further 5.0 to 14.3 points, reaching 87.4 and 89.5 percent on MultiTQ, 71.1 and 70.4 percent on CronQuestion, and 71.7 and 69.9 percent on ICEWS-Actor.}
  \label{fig:stages}
\end{figure}

Figure~\ref{fig:stages} reports the accuracy of \system after each training stage.
The base policy is the instruction-tuned backbone without training, which performs both roles with the same prompts and tools.
The distilled policy is obtained by role-conditioned distillation, and the self-play policy is \system itself.

Role-conditioned distillation provides the larger gain in five of the six settings.
It improves Hits@1 by 10.3 to 35.9 points over the base policy, with the largest gain for Llama on MultiTQ (from 39.3\% to 75.2\%).
This is because the base policy does not reliably perform subtask delegation, tool invocation, and message exchange, as discussed in Section~\ref{sec:initialization}.
Distillation teaches both roles these behaviors from complete collaborations.

Co-evolving policy optimization further improves all benchmarks.
It adds 5.0 to 14.3 points over the distilled policy, reaching 87.4\% for Qwen and 89.5\% for Llama on MultiTQ.
This is because the distilled policy only imitates teacher-generated collaborations, whereas self-play trains both roles on the final outcomes of their own collaborations.

Overall, role-conditioned distillation enables the shared policy to collaborate, whereas co-evolving policy optimization improves how the coordinator and workers collaborate.

\subsection{Analysis of Process Reward}
\label{app:process_reward}

\begin{figure}[t]
  \centering
  \includegraphics[width=\columnwidth]{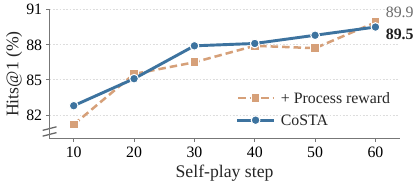}
  \caption{Hits@1 (\%) of Llama on MultiTQ during self-play with and without the process reward on worker reports.}
  \Description{Two lines over self-play steps 10 to 60. CoSTA with the final-answer reward reaches 82.8, 85.1, 87.9, 88.1, 88.8, and 89.5 percent. Adding the process reward gives 81.2, 85.5, 86.5, 87.9, 87.7, and 89.9 percent, lower than CoSTA at four of the six steps.}
  \label{fig:process_reward}
\end{figure}

Co-evolving policy optimization rewards each collaboration only by its final answer.
To examine whether workers benefit from additional supervision, we add a process reward on worker reports.
DeepSeek-V4-Flash serves as the judge and scores each report without access to the gold answer.
The score is added to the advantages of worker tokens with a weight of 0.2.
We train Llama on MultiTQ under the same settings as \system.

The final-answer reward is sufficient for co-evolving policy optimization.
As shown in Figure~\ref{fig:process_reward}, the process reward lowers the accuracy at four of the six steps, and the final accuracy differs by only 0.4 points (89.9\% vs. 89.5\%).
This is because cross-role credit assignment already propagates the final outcome to worker decisions, as shown in Appendix~\ref{app:theory}, whereas the judge evaluates each report without the subsequent evidence synthesis.
In addition, the process reward requires an LLM call for every worker report during training.

\section{Case Study}
\label{app:case_study}

\begin{table}[t]
\centering
\caption{Case study on TimelineKGQA-ICEWS-Actor with Qwen3-4B.}
\label{tab:case_study}
\footnotesize
\renewcommand{\arraystretch}{1.15}
\begin{tabularx}{\columnwidth}{@{}l X@{}}
\toprule
\multicolumn{2}{@{}p{\columnwidth}@{}}{
\textbf{Q:} From when to when, Network of East West Women Poland
Affiliation To Human Rights IGOs (Non-Governmental Organizations),
at the same time, Stipe Suvar Affiliation To Government (Serbia),
at the same time, István Tarlós Affiliation To Fidesz?
}\\
\multicolumn{2}{@{}l@{}}{
\textbf{Gold:} 1988-06-30 to 1989-05-17
}\\
\midrule
\multicolumn{2}{@{}l@{}}{\textbf{\system}}\\
Worker 1 & $(-\infty, +\infty)$\\
Worker 2 & [1988-06-30, 1989-05-17]\\
Worker 3 & $(-\infty, +\infty)$, read from Tarlós's timeline\\
Coordinator &
Intersection: 1988-06-30 to 1989-05-17
\textcolor{oursink}{\ding{51}}\\
\midrule
\multicolumn{2}{@{}l@{}}{\textbf{Single agent}}\\
Tarlós--Fidesz &
[2006-01-01, $+\infty$), the first search result\\
Answer & No Answer \textcolor{red}{\ding{55}}\\
\bottomrule
\end{tabularx}
\end{table}

Table~\ref{tab:case_study} shows a question that requires intersecting three time intervals.
\system assigns each fact to a separate worker.
By reading the timeline of István Tarlós, the third worker obtains the complete interval of his Fidesz affiliation.
The single agent investigates all three facts in one context and adopts the interval starting in 2006 from its first search result, which leads to an empty intersection.

\end{document}